\documentclass[journal]{IEEEtran}

\usepackage{amsxtra, amsfonts, amssymb, amstext, amsmath, mathtools}
\usepackage{amsthm}
\usepackage{booktabs}
\usepackage{nicefrac}
\usepackage{xspace}
\usepackage{graphicx}
\usepackage[ruled,vlined,linesnumbered]{algorithm2e}
\SetArgSty{} 

\usepackage{dsfont}
\usepackage[shortcuts]{extdash}
\usepackage{enumitem}
\usepackage{todonotes}

\usepackage{bbm}
\usepackage[utf8]{inputenc}
\usepackage{hyperref}
\usepackage{cleveref}

\crefname{prop}{property}{properties}

\allowdisplaybreaks

\newcommand{\oea}{\mbox{${(1 + 1)}$~EA}\xspace}

\newcommand{\ooea}{\oea}

\newcommand{\oplea}{\mbox{${(1+\lambda)}$~EA}\xspace}

\newcommand{\onemax}{\textsc{OneMax}\xspace}
\newcommand{\LO}{\textsc{Leading\-Ones}\xspace}

\newcommand{\leadingones}{\LO}

\newcommand{\binval}{\textsc{BinVal}\xspace}

\newcommand{\jump}{\textsc{Jump}\xspace}

\newcommand{\R}{\ensuremath{\mathbb{R}}}

\newcommand{\N}{\ensuremath{\mathbb{N}}} 

\newcommand{\ind}[1]{\mathbf{1}\left\{#1\right\}}

\let\originalleft\left
\let\originalright\right
\renewcommand{\left}{\mathopen{}\mathclose\bgroup\originalleft}
\renewcommand{\right}{\aftergroup\egroup\originalright}
\DeclareMathOperator{\E}{\mathds{E}}

\newtheorem{theorem}{Theorem}
\newtheorem{lemma}[theorem]{Lemma}

\begin{document}

\title{Runtime Analysis of the Compact Genetic Algorithm on the LeadingOnes Benchmark}

\author{Marcel Chwiałkowski\thanks{M.~Chwiałkowski is with the {\'E}cole Polytechnique, Institut Polytechnique de Paris, Palaiseau, France}, Benjamin Doerr\thanks{B.~Doerr is with the {\'E}cole Polytechnique, Institut Polytechnique de Paris, CNRS, Laboratoire d'Informatique (LIX), Palaiseau, France}, and %
  Martin S. Krejca\thanks{M.~S. Krejca is with the {\'E}cole Polytechnique, Institut Polytechnique de Paris, CNRS, Laboratoire d'Informatique (LIX), Palaiseau, France}%
  \thanks{This research benefited from the support of the FMJH Program Gaspard Monge for optimization and operations research and their interactions with data science.}
}

\markboth{IEEE TRANSACTIONS ON EVOLUTIONARY COMPUTATION, AUTHOR-PREPARED MANUSCRIPT}%
{IEEE TRANSACTIONS ON EVOLUTIONARY COMPUTATION, AUTHOR-PREPARED MANUSCRIPT}

\maketitle

\begin{abstract}
  The compact genetic algorithm (cGA) is one of the simplest estimation-of-distribution algorithms (EDAs).
  Next to the univariate marginal distribution algorithm (UMDA)---another simple EDA---, the cGA has been subject to extensive mathematical runtime analyses, often showcasing a similar or even superior performance to competing approaches.
  Surprisingly though, up to date and in contrast to the UMDA and many other heuristics, we lack a rigorous runtime analysis of the cGA on the \leadingones benchmark---one of the most studied theory benchmarks in the domain of evolutionary computation.

  We fill this gap in the literature by conducting a formal runtime analysis of the cGA on \leadingones.
  For the cGA's single parameter---called the hypothetical population size---at least polylogarithmically larger than the problem size, we prove that the cGA samples the optimum of \leadingones with high probability within a number of function evaluations quasi-linear in the problem size and linear in the hypothetical population size.
  For the best hypothetical population size, our result matches, up to polylogarithmic factors, the typical quadratic runtime that many randomized search heuristics exhibit on \leadingones.
  Our analysis exhibits some noteworthy differences in the working principles of the two algorithms which were not visible in previous works.
\end{abstract}

\begin{IEEEkeywords}
  Estimation-of-distribution algorithms, compact genetic algorithm, runtime analysis, LeadingOnes.
\end{IEEEkeywords}

\section{Introduction}
\label{sec:introduction}

\IEEEPARstart{E}{stimation}-of-distribution algorithms~\cite{PelikanHL15}~(EDAs) constitute a large and important class of general-purpose randomized optimization heuristics (ROHs)~\cite{KacprzykP2015}.
A key feature of EDAs that clearly separates them from other prominent classes of ROHs, such as evolutionary algorithms~\cite{EibenS15}, is that they maintain and evolve a probabilistic model of the search space.
This model defines a probability distribution over the search space and aims at giving better solutions a higher probability.
Ideally, only global optima in the search space have a positive probability.
EDAs iteratively refine their model based on random samples thereof, giving more probability to more promising solutions.
This approach has been shown to outperform competing approaches both empirically~\cite{PelikanHL15} as well as theoretically~\cite{KrejcaW20bookchapter} on numerous occasions.

Like other ROHs, EDAs have been subject to mathematical investigations for more than two decades, starting with convergence guarantees~\cite{KrejcaW20bookchapter}.
The first fully rigorous runtime analysis comparable to the current state of the art in this field was carried out by Droste~\cite{Droste06}.
In this seminal paper, he analyzed the \emph{compact genetic algorithm}~\cite{HarikLG99} (cGA) on the \onemax benchmark.
The cGA is a simple EDA that samples the problem variables independently, resulting in a so-called \emph{univariate} model.
It refines this model in each iteration based solely on two independent samples, shifting probability mass in discrete steps of size~$\frac{1}{\mu}$, where~$\mu$ is an algorithm-specific parameter called the \emph{hypothetical population size}.
Since the cGA is a simple EDA and since \onemax is the most commonly studied theory benchmark in the field of evolutionary computation~\cite{DoerrN20}, the result by Droste~\cite{KrejcaW20bookchapter} served as a logical and strong first runtime contribution in the domain of EDAs.
Interestingly, this paper was initially not followed up by a lot of theoretical work.

The rigorous investigation of EDAs only gained momentum with the paper by Dang and Lehre~\cite{DangL15gecco}, who analyzed the runtime of the \emph{univariate marginal distribution algorithm}~\cite{MuhlenbeinP96} (UMDA), which is also a univariate EDA but allows, in contrast to the cGA, for a larger sample size than two.
Dang and Lehre considered the \onemax and the \leadingones theory benchmarks, covering the two most important theory benchmarks in the field.
Since then, numerous theoretical results on EDAs have been published~\cite{KrejcaW20bookchapter}.
Most of these results still only consider the cGA and the UMDA, but they provide us with deep insights into the complex dynamics of updating the algorithms' probabilistic models.
Moreover, they cover more complex benchmarks than \onemax and \leadingones, and they also consider more complex scenarios, such as noisy optimization.
In \Cref{sec:relatedWork}, we give more detail about some of these works.

Surprisingly, despite this increased theoretical body of literature and the long way that runtime analysis on EDAs has come since the initial paper by Droste~\cite{KrejcaW20bookchapter}, to this day, the cGA's runtime performance on the \leadingones benchmark has not been rigorously studied.
In other words, we are still missing a theoretical result for one of the most studied EDAs on one of the most studied theory benchmarks.
In comparison, runtime guarantees for the commonly studied UMDA on \leadingones have existed for more than a decade~\cite{ChenTCY10}.

\smallskip
\textbf{Our contribution:}
We conduct the first runtime analysis of the cGA on \leadingones.
We consider the parameter regime of the algorithm where it exhibits low \emph{genetic drift}, that is, the updates of the cGA to its probabilistic model are sufficiently small such that random fluctuations in the update process do not have a strong impact on the overall performance.
For a problem size of~$n$, we show that the cGA with $\mu = \Omega(n \log^2 n)$ samples the optimum of \leadingones with high probability after $O(\mu n \log n)$ function evaluations (\Cref{thm:cgaOnLO}).
This runtime bound is minimized for $\mu = \Theta(n \log^2 n)$, resulting in an overall runtime of $O(n^2 \log^3 n)$.

This runtime bound is slightly worse, by a factor of $O(\log^3 n)$, compared to the common $O(n^2)$ bound of many other ROHs on \leadingones~\cite{DoerrN20}, including the best-known runtime bound of the UMDA~\cite{DangLN19}.
For the generally preferred regime of low genetic drift, the best known runtime bound for the UMDA is $\Theta(n^2 \log n)$~\cite{DoerrK21tcs}, which is only better by a factor of $O(\log^2 n)$.
These small differences might be a consequence of how the cGA and the UMDA fix certain parts in their probabilistic model.
As the UMDA, due to its larger sample size, can apply larger (justified) changes to its model, it can adjust it slightly quicker. We discuss this difference in the optimization behavior of the two algorithms in more detail after Lemma~\ref{lem:increaseCriticalFrequency} and a second difference at the end of Section~\ref{sec:runtimeAnalysis}.

We note, though, that we currently have no matching lower bound. Hence we cannot say whether the two algorithms truly have a different asymptotic runtime or the difference in the upper bounds is caused by our upper bound not being absolutely tight. Due to the difficulty of proving lower bounds, we have to leave this an open problem.

\smallskip
\textbf{Outline:}
We discuss related work in \Cref{sec:relatedWork}.
In \Cref{sec:preliminaries}, we describe the cGA, the \leadingones problem, and the mathematical tools we use for our analysis.
We then conduct the runtime analysis of the cGA on \leadingones in \Cref{sec:runtimeAnalysis}.
Last, we conclude our paper in \Cref{sec:conclusion}.

\section{Related Work}
\label{sec:relatedWork}

We now discuss important runtime results of EDAs that relate to our result.
The term \emph{runtime} refers to the number of function evaluations until a global maximum of the objective function is sampled for the first time.
For EDAs, runtime results are often shown to hold with high probability but also regularly in expectation.
We do not detail which is the case, as they can often be translated into each other.
Moreover, we note that EDAs are often studied with \emph{borders} in their probabilistic model, which prevent them from fully fixing problem variables to a deterministic value.
If borders are absent, the translation between expected runtime and the ones with high probability no longer works, as the runtime of the algorithm is infinite if the probabilistic model fixes a problem variable to a wrong value.
Moreover, results in the regime of high genetic drift only hold if borders are present.
However, since we aim to provide an overview on the general runtime behavior of the algorithms, we do not mention the specific details in the following and refer instead to the respective works.

We first discuss the effect of \emph{genetic drift} in EDAs, which underlies all of the runtime results, although not always explicitly stated this way.
Afterward, we discuss results of the cGA on benchmarks other than \leadingones.
Last, we discuss results of other EDAs on \leadingones.

\smallskip
\textbf{Genetic drift in EDAs.}
Genetic drift is a concept that quantifies how much an EDA is affected by random fluctuations in the model update, due to the variance in the samples.
If genetic drift is \emph{high}, this usually implies that the probability is high that the probability mass for an incorrect value of some problem variables\footnote{That is, a value that is different from the one of a global optimum.} is close to~$1$.
Analogously, \emph{low} genetic drift usually implies with high probability that the probability mass for each problem variable is not large for incorrect values, for a sufficiently long time.

Doerr and Zheng~\cite{DoerrZ20tec} provide a tight quantification of the effect of genetic drift in various EDAs, including the cGA and the UMDA.
Their results in particular allow to argue that with high probability the probability mass for incorrect values of the problem variables remains small.
During this time, an EDA has usually the best chances of succeeding in optimization.

\smallskip
\textbf{Runtime results for the cGA.}
A substantial line of research studies the runtime of the cGA on the \onemax benchmark, which returns the number of~$1$s in a length-$n$ bit strings; a quantity that is aimed to be maximized.
Many classic ROHs with suitable parameter values have an expected runtime of $O(n \log n)$ on \onemax~\cite{Muhlenbein92,DrosteJW02,JansenJW05,Witt06,JansenW07,RoweS14,AntipovD21algo,DoerrERW23}.
The seminal paper by Droste~\cite{Droste06} came close to this bound for the cGA, proving for an arbitrary constant $\varepsilon > 0$ and for a choice of $\mu = \Omega(n^{1/2 + \varepsilon})$ a runtime of $O(\mu \sqrt{n})$.
This result was improved and made tight by Sudholt and Witt~\cite{SudholtW19}, who proved an upper bound of $O(\mu \sqrt{n})$ for $\mu = \Omega(\sqrt{n} \log n)$, and a general lower bound of $\Omega(\mu \sqrt{n} + n \log n)$, matching the typical $O(n \log n)$ bound for the best choice of~$\mu$.
We note that the upper bound pertains to the regime of low genetic drift.
This result was later complemented by Lengler, Sudholt, and Witt~\cite{LenglerSW21}, who considered smaller values of~$\mu$, where the effect of genetic drift is more prominent, and showed a bound of $\Omega(\mu^{1/3} n + n \log n)$ for $\mu = O(\sqrt{n} / \log^2 n)$.
This result shows that the cGA has a runtime strictly worse than $O(n \log n)$ for $\mu = \omega(\log^3 n) \cap O(\sqrt{n} / \log^2 n)$, a parameter range in the regime with high genetic drift.

The cGA has also been analyzed on a noisy version of \onemax where additive posterior noise disturbs the fitness~\cite{FriedrichKKS17}. This result shows that the cGA can deal with noise of very large variance if the parameter~$\mu$ is chosen adequately, which is in contrast to population-based evolutionary algorithms (EAs), which fail.
Using a smart restart strategy~\cite{ZhengD23jmlr}, this performance can also be achieved without knowing the variance of the noise.
A similar superior performance of the cGA was also shown for a rugged version of \onemax~\cite{FriedrichKNR22}.

The cGA also shows a very promising runtime behavior in overcoming local optima.
Hasenöhrl and Sutton~\cite{HasenohrlS18} initiated the runtime analysis of the cGA on the \jump benchmark with gap size~$k$, proving superpolynomial speed-ups compared to mutation-only EAs.
This result was later refined by Doerr~\cite{Doerr21cgajump}, who proved that the cGA optimizes \jump with $k = O(\log n)$ in only $O(\mu \sqrt{n})$ function evaluations if $\mu = \Omega(\sqrt{n} \log n) \cap \mathrm{poly}(n)$.
This results for $\mu = \Theta(\sqrt{n} \log n)$ in a runtime of $O(n \log n)$, which is the same as on \onemax, standing in stark contrast to mutation-only EAs, who exhibit a runtime of $\Theta(n^k)$~\cite{JansenW02}.
Moreover, Witt~\cite{Witt23} showed that the cGA also performs well with different variants of fitness valleys, whereas other approaches that perform well on \jump fail.

Last, we mention that the cGA was also studied on the binary-value function \binval.
The first result was given again by Droste~\cite{Droste06}, proving for arbitrary $\varepsilon > 0$ and a choice of $\mu = \Omega(n^{1 + \varepsilon})$ a runtime bound of $O(\mu n)$, as well as the general bound $\Omega(\mu n)$.
Witt~\cite{Witt18} refined this result and proved for a choice of $\mu = \Omega(n \log n) \cap \mathrm{poly}(n)$ a runtime bound of $\Theta(\mu n)$.
This runtime is strictly worse than the general $O(n \log n)$ bound that the \ooea exhibits on all linear functions~\cite{Witt13}, showing that the cGA does not always compete or outperform other existing approaches. That more complex algorithms may find it harder to optimize \binval was later observed also for the \oplea~\cite{DoerrK15}, which for larger population sizes is again is faster on \onemax than on \binval.

\smallskip
\textbf{EDA runtime results for \leadingones.}
\leadingones returns the length of the longest all-$1$s prefix of a bit string.
It has been mostly studied for the UMDA, which takes~$\lambda$ samples each iteration and performs its update based on the~$\mu$ best of these.
We note that the dynamics of the UMDA can depend drastically on the ratio of~$\mu$ and~$\lambda$ (the \emph{selection pressure}).
However, in order to keep this section brief, we assume that $\mu / \lambda$ is constant in the problem size and sufficiently well-chosen, without elaborating on the details in the following.
We note that under these assumptions,~$\mu$ of the UMDA relates well to~$\mu$ of the cGA when comparing runtimes.

The first result of the UMDA on \leadingones was proven by Chen~et~al.~\cite{ChenTCY10}, showing for an arbitrary constant $\varepsilon > 0$ and a choice of $\mu = \Omega(n^{2 + \varepsilon})$ a runtime of $O(\mu n)$, far off from the common runtime bound of $O(n^2)$ of other ROHs~\cite{Rudolph97,DrosteJW02,JansenJW05,Witt06,BottcherDN10,Sudholt13,Doerr19tcs}.
Dang and Lehre~\cite{DangL15gecco} drastically improved this result by showing a runtime of $O(n \mu \log \mu + n^2)$ for $\mu = \Omega(\log n)$.
For $\mu = O(n / \log n)$, this matches the common $O(n^2)$ bound.
Interestingly, this result applies to a regime of~$\mu$ where genetic drift is not low.
The case with low genetic drift, for $\mu = \Omega(n \log n)$, was refined by Doerr and Krejca~\cite{DoerrK21tcs}, who proved a runtime of $\Theta(\mu n)$.
For the best choice of~$\mu$, this results in a bound of $\Theta(n^2 \log n)$, which is worse than the typical $O(n^2)$ runtime by a factor of $O(\log n)$.
This shows that the regime with low genetic drift is not optimal for the UMDA on \leadingones.
An analogous asymptotically tight bound was recently proven for a multi-valued version of the UMDA on a multi-valued \leadingones version~\cite{BenJedidiaDK24}, which generalizes the binary case.
We note that Adak and Witt~\cite{AdakW24} used the same framework and proved an upper bound of a multi-valued cGA on a multi-valued \onemax version, almost analogous to the bound in the binary domain for low genetic drift, but worse by a factor of~$\log^2 n$.

The UMDA was also considered in a prior-noise setting on \leadingones~\cite{LehreN21}, where, with constant probability, instead of returning the function value of a sample $x \in \{0, 1\}^n$, the noisy version returns the function value of a variant of~$x$ where a single position chosen uniformly at random is inverted.
The authors prove that the runtime of the UMDA in this scenario is identical to the unnoisy case.
That is, for $\mu = \Omega(\log n)$, the runtime is $O(n \mu \log \mu + n^2)$.

\leadingones was also studied for newly created EDAs that operate differently from purely univariate EDAs like the cGA and the UMDA.
Doerr and Krejca introduced the \emph{significance-based} cGA~\cite{DoerrK20tec} (sig-cGA), which is similar to the traditional cGA but also stores some history of good variable values 
and only updates its probabilistic model if this history shows a statistically significant bias toward a certain value. This algorithm does not require the parameter~$\mu$ of the classic cGA.
The authors prove a runtime of $O(n \log n)$ of the sig-cGA on \leadingones, which was the first ROH with such a good proven runtime on \leadingones while also showcasing a runtime of $O(n \log n)$ on \onemax.
This result was followed by the introduction of the \emph{competing genes evolutionary algorithm} (cgEA) by Ajimakin and Devi~\cite{AjimakinD22}, which is an algorithm with a population of size~$\mu$ that fixes in each iteration the value of one problem variable.
The cgEA bases this decision on a more involved decision process, making use of the Gauss--Southwell score and of inverting values at specific positions.
The authors prove for $\mu = \Omega(\log n)$ a runtime of $O(\mu n)$ of the cgEA on \leadingones, which matches the $O(n \log n)$ bound of the sig-cGA for the best choice of~$\mu$.

\section{Preliminaries}
\label{sec:preliminaries}

Let~$\N$ denote the natural numbers, including~$0$, and for all $m, n \in \N$, let $[m .. n] = [m, n] \cap \N$ as well as $[n] \coloneqq [1 .. n]$.

Let $n \in \N_{\geq 1}$.
We consider pseudo-Boolean maximization, that is, the maximization of \emph{fitness functions} $f\colon \{0, 1\}^n \to \R$.
We call each $x \in \{0, 1\}^n$ an \emph{individual}, and~$f(x)$ the \emph{fitness of~$x$}.
Moreover, for all $v \in \R^n$ and all $i \in [n]$, let~$v_i$ denote the value of~$v$ at position~$i$.
When given a fitness function, we implicitly assume that~$n$ is provided, and when we use big-O notation, it always refers to asymptotics in this~$n$.
For example, an expression of $o(1)$ tends to~$0$ as~$n$ tends to infinity.

We say that an event~$E$ occurs \emph{with high probability} if and only if the probability of~$E$ not occurring is at most $O(n^{-1})$.
Last, for each proposition~$A$, we let~$\ind{A}$ denote the indicator function for~$A$, which is~$1$ if~$A$ is true, and it is~$0$ otherwise.

\subsection{The Compact Genetic Algorithm}
\label{sec:preliminaries:cga}

The \emph{compact genetic algorithm}~\cite{HarikLG99} (cGA, \Cref{alg:cga}) maximizes a given pseudo-Boolean function~$f$ by maintaining a \emph{frequency vector} $p \in [\frac{1}{n}, 1 - \frac{1}{n}]^n$ that corresponds to a probability distribution over $\{0, 1\}^n$ as follows. For each $i \in [n]$, the \emph{frequency}~$p_i$ indicates the independent probability to sample a~$1$ at position~$i$.
Formally, let $x \in \{0, 1\}^n$, and let $Z \in \{0, 1\}^n$ be a random sample from~$p$.
Then $\Pr[Z = x] = \prod_{i \in [n]} \bigl(p_i^{x_i} (1 - p_i)^{1 - x_i}\bigr)$.
Note that since~$p$ is restricted to $[\frac{1}{n}, 1 - \frac{1}{n}]^n$, at any time any variable value can be sampled with positive probability.

The cGA updates~$p$ iteratively based on two independent samples, $x^{(1)}$ and~$x^{(2)}$, and on an algorithm-specific parameter $\mu \in \R_{> 0}$ called the \emph{hypothetical population size}.
It evaluates the fitness of the two samples and assigns the sample with the higher fitness to~$y^{(1)}$ and the other one to~$y^{(2)}$.
Afterward, it adjusts the components of~$p$ where $y^{(1)}$ and $y^{(2)}$ differ by an additive \emph{step size} of~$\frac{1}{\mu}$ towards the value of~$y^{(1)}$. More precisely, for each $i \in [n]$, the following update is performed.
If $y^{(1)}_i = 1$ and $y^{(2)}_i = 0$, then~$p_i$ is increased by~$\frac{1}{\mu}$.
If $y^{(1)}_i = 0$ and $y^{(2)}_i = 1$, then~$p_i$ is decreased by~$\frac{1}{\mu}$.
If $y^{(1)}_i = y^{(2)}_i$, then~$p_i$ remains unchanged.
Last, for all positions $i \in [n]$, if~$p_i$ is less than~$\frac{1}{n}$, it is set to~$\frac{1}{n}$.
Analogously, if~$p_i$ is greater than $1 - \frac{1}{n}$, it is set to $1 - \frac{1}{n}$.
We say that~\emph{$p$ is restricted to $[\frac{1}{n}, 1 - \frac{1}{n}]^n$}.

Due to restricting~$p$, the update made to each component can be less than~$\frac{1}{\mu}$ while larger than~$0$.
As this makes the mathematical analysis needlessly complicated, we assume that the range $\frac{1}{2} - \frac{1}{n}$ of frequency values from~$\frac{1}{2}$ to either of the extreme values is a multiple of the step size~$\frac{1}{\mu}$.
This ensures that non-zero updates to~$p$ are always exactly~$\frac{1}{\mu}$.
This is known as the \emph{well-behaved frequency assumption}~\cite{Doerr21cgajump}, and we apply this terminology also to~$\mu$, meaning that its choice results in well-behaved frequencies.
This assumption solely simplifies the mathematical analysis.
All of our results are still valid for other values of~$\mu$ that respect the restrictions of the respective mathematical statement.

\Cref{alg:cga} uses the same notation as above but indicates the iteration $t \in \N$ for each variable in the superscript.
We say that the algorithm is in iteration~$t$ if and only if the samples are indexed with~$t$.
Hence, the first iteration is~$0$.
We are interested in the number of iterations until a global optimum of~$f$ is sampled for the first time.
We call this value the \emph{runtime of the cGA on~$f$}.
We note that since the cGA evaluates~$f$ exactly twice each iteration, its runtime is also the number of fitness evaluations until a global optimum of~$f$ is sampled for the first time, up to a factor of~$2$.

\smallskip
\textbf{Differences to practical implementations.}
\Cref{alg:cga} does not mention a specific termination criterion.
As we stop our theoretical considerations once a global optimum is sampled, we require the cGA to run at least as long, which is trivially given if it runs indefinitely.
In practice, a budget for the number of fitness evaluations is typically employed.
Similarly, we do not specify what the cGA actually returns, as this is irrelevant for our analysis.
In practice, since one is interested in good individuals, one usually stores the individual with the best-so-far fitness and returns it together with the frequency vector upon termination.

\begin{algorithm}[t]
  \caption{The \emph{compact genetic algorithm}~\cite{HarikLG99} with well-behaved \emph{hypothetical population size} $\mu \in \R_{> 0}$, maximizing a given pseudo-Boolean function~$f$.}
  \label{alg:cga}
  $t \gets 0$\;
  $p^{(t)} \gets (\frac{1}{2})_{i \in [n]}$\;
  \Repeat{termination criterion met}{%
    $x^{(1, t)} \gets$ sample of~$p^{(t)}$\;
    $x^{(2, t)} \gets$ sample of~$p^{(t)}$\;
    $y^{(1, t)} \gets x^{(1, t)}$\;
    $y^{(2, t)} \gets x^{(2, t)}$\;
    \If{$f\bigl(x^{(1, t)}\bigr) < f\bigl(x^{(2, t)}\bigr)$}{%
      swap~$y^{(1, t)}$ and~$y^{(2, t)}$\;
    }
    \ForEach{$i \in [n]$}{%
      $p^{(t + 1)}_i \gets p^{(t)}_i + \frac{1}{\mu}\bigl(y^{(1, t)}_i - y^{(2, t)}_i\bigr)$\;
    }
    restrict~$p^{(t + 1)}$ to $[\frac{1}{n}, 1 - \frac{1}{n}]^n$\;
    $t \gets t + 1$\;
  }
\end{algorithm}

\subsection{The LeadingOnes Benchmark}
\label{sec:praliminaries:leadingOnes}

\leadingones, first proposed in~\cite{Rudolph97}, is one of the most prominent benchmarks in the theoretical analysis of randomized search heuristics. 
The problem returns the longest prefix of consecutive~$1$s in an individual.
Formally,
\begin{align*}
  \leadingones\colon x \mapsto \max \{i \in [0 .. n]\colon \forall j \in [i]\colon x_j = 1\}.
\end{align*}
The problem is unimodal, with the all-$1$s bit string being the global optimum.
Note that if an individual has a \leadingones fitness of $k \in [0 .. n]$, then the bit at position $k + 1$ is a~$0$ (if this position exists), and the bit values at positions larger than $k + 1$ do not contribute to the fitness at all.

\subsection{Mathematical Tools}
\label{sec:preliminaries:tools}

We carefully study how each frequency is updated over time in expectation, and then translate these values into bounds on hitting times for each frequency.
The mathematical technique that formalizes this translation is called \emph{drift analysis}~\cite{Lengler20bookchapter}, with \emph{drift} referring to the expected change of the process.
In particular, we make use of the following three theorems.
We adjust their formulation to better suit our presentation.

The first theorem shows that it is unlikely for a random process whose progress is linearly proportional to its expected progress to not reach a certain value after a short time.
This so-called \emph{multiplicative-drift theorem} dates back to Doerr, Johannsen, and Winzen~\cite{DoerrJW12algo}, with the concentration bound we utilize being first shown by Doerr and Goldberg~\cite{DoerrG13algo}.

\begin{theorem}[{Multiplicative drift~\cite[Theorem~$18$]{Lengler20bookchapter}}]
  \label{thm:multiplicativeDrift}
  Let $(X_t)_{t \in \N}$ be a random process over $\mathcal{S} \subseteq \R_{\geq 0}$ with $0 \in \mathcal{S}$.
  Let $s_{\mathrm{min}} = \min(S \setminus \{0\})$, and let $T = \inf \{t \in \N \mid X_t = 0 \}$.
  Assume that there is a $\delta \in \R_{> 0}$ such that for all $t \in \N$ it holds that
  \begin{align*}
    \E[X_t - X_{t + 1} \mid X_t] \cdot \ind{t < T} \geq \delta X_t \cdot \ind{t < T}.
  \end{align*}
  Then for all $r \in \R_{\geq 0}$, we have
  \begin{align*}
    \Pr\left [ T >\left  \lceil \frac{r + \ln(X_0/s_{\mathrm{min}})}{\delta} \right \rceil \,\middle\vert\, X_0\right ] \leq \exp(-r).
  \end{align*}
\end{theorem}

The following theorem shows that a random process that moves in expectation away from a target point $b \in \R_{> 0}$ has a low probability of hitting it within a number of steps polynomial in~$b$.
The variant of this \emph{negative-drift} theorem is due to Kötzing~\cite{Kotzing16}, who phrased it in a fashion that required the process to move away from the target in expectation at any point in time.
The following version only requires this for an interval.

\begin{theorem}[{Negative drift~\cite[Corollary~$3.24$]{Krejca19}}]
  \label{thm:negativeDrift}
  Let $(X_t)_{t\in \mathbb{N}}$ be a random process over~$\R$.
  Moreover, let $X_0 \leq 0 $, let $b \in \R_{> 0}$, and let $T = \inf\{t \in \N \mid X_t \geq b \}$.
  Suppose that there are values $a \in \R_{\leq 0}$, $c \in (0,b)$, and $\varepsilon \in \R_{< 0}$ such that for all $t \in \N$ holds that
  \begin{enumerate}[label=(\alph*)]
    \item $\E[X_{t + 1} - X_t \mid X_t ] \cdot \ind{X_t \geq a} \leq \varepsilon \cdot \ind{X_t \geq a}$,
    \item $|X_t - X_{t + 1}| \cdot \mathbf{1}\{X_t \geq a\} < c$, and
    \item $X_{t+ 1} \cdot \mathbf{1} \{X_t < a \} \leq 0.$
  \end{enumerate}
  Then for all $t \in \mathbb{N}$, we have
  \begin{align*}
    \Pr[T \leq t] \leq t^2 \cdot \exp\left(-\frac{b|\varepsilon|}{2c^2}\right).
  \end{align*}
\end{theorem}

The last theorem bounds the effect of genetic drift.
That is, it shows that frequencies in the cGA do not get too low too quickly if the fitness function that is being optimized considers at each position a~$1$ at least as valuable as a~$0$ in terms of fitness.
Formally, a fitness function~$f$ \emph{weakly prefers~$1$s over~$0$s at position $i \in [n]$} if and only if for all $x, y \in \{0, 1\}^n$ that differ only in position~$i$ such that $x_i = 1$ (and $y_i = 0$), we have $f(x) \geq f(y)$.

\begin{theorem}[{Genetic drift~\cite[Corollary~$2(2)$]{DoerrZ20tec}}]
  \label{thm:geneticDrift}
  Consider the cGA with a well-behaved hypothetical population size $\mu \in \R_{> 0}$ optimizing a fitness function~$f$ that weakly prefers~$1$s over~$0$s at position $i \in [n]$.
  Then for all $\gamma \in \R_{> 0}$ and all $T \in \N$ it holds that
  \begin{align*}
    \Pr\left[\forall t \in [0..T]\colon p_i^{(t)} > \frac{1}{2} - \gamma\right] \geq 1 - 2\exp\left(-\frac{\gamma^2\mu^2}{2T}\right).
  \end{align*}
\end{theorem}

\section{Runtime Analysis of the cGA on LeadingOnes}
\label{sec:runtimeAnalysis}

Our main result is \Cref{thm:cgaOnLO}, which essentially proves that the cGA with $\mu = \Omega(n \log^2 n)$ maximizes \leadingones with high probability in $O(\mu n \log n)$ iterations.
In particular, this runtime is minimized for $\mu = \Theta(n \log^2 n)$ and results in a runtime of $O(n^2 \log^3 n)$ with high probability.
This bound is only worse by a factor of $O(\log^3 n)$ than the common $O(n^2)$ runtime bound of many other randomized search heuristics~\cite{DoerrN20}, and it is only worse by a factor of $O(\log^2 n)$ than the runtime of the EDA UMDA for a comparable parameter regime~\cite{DoerrK21tcs}.
Notably, the cGA achieves this result while only having access to two  samples in each iteration.

\begin{theorem}
  \label{thm:cgaOnLO}
  Consider the cGA maximizing \LO with $n \geq 8$ and $\mu \geq 786\exp(6)n \ln^2 n$.
  Then with a probability of at least $1 - \frac{4}{n}$, after $12\exp(6)\mu n \ln n$ iterations, all frequencies are greater than $1 - \frac{3}{n}$ and remain that way for at least the next $12\exp(6)\mu n \ln n$ iterations.
  Moreover, with a probability of at least $1 - \frac{5}{n}$, the algorithm finds the optimum of \LO within $24\exp(6)\mu n \ln n$ iterations.
\end{theorem}

We note that we state and prove our main result only for $n \ge 8$. By excluding the small values $n \le 7$, we can work with explicit (though not optimized) constants, which increases the readability of the proof. We note that, trivially, the asymptotic statements made in the introduction hold for $n \le 7$; for this it suffices to see that from any state of the algorithm, with constant probability in a constant number of iterations the frequency vector $(1-\frac 1n, \dots, 1 - \frac 1n)$ is reached, namely by always sampling the all-ones and the all-zeros vector.

We now turn to the more interesting case that $n$ is large, for which $n \ge 8$ suffices here. Our analysis considers the cGA in the parameter regime with low \emph{genetic drift}.
Genetic drift refers to the effect that a frequency at position $i \in [n]$ is shifted over time based solely on the random order of the bit values at position~$i$ of the samples, that is, without the fitness function biasing the ranking.
This occurs regularly when optimizing \leadingones, as bits after the first~$0$ do not affect the ranking of the samples.
This leads to random updates of frequencies.
However, in the regime of \emph{low} genetic drift, this effect is negligible (\Cref{lem:frequenciesDoNotDropTooLow}).
This is achieved by choosing~$\mu$ to be sufficiently large.
Formally, we rely on \Cref{thm:geneticDrift}, which bounds with high probability the total negative offset a frequency receives.

By relying on the impact of genetic drift being low, our analysis proceeds by showing that the frequencies are increased sequentially until they reach values close to $1 - \frac{1}{n}$.
Moreover, while it is unlikely for a frequency to remain exactly at $1 - \frac{1}{n}$, we show that it does not go below $1 - \frac{3}{n}$ with high probability before the optimum is sampled (\Cref{lem:frequencyStayingHigh}).
This motivates the definition of the \emph{critical position} $i \in [n]$, which is (for each iteration) the smallest position in $[n]$ whose frequency is below $1 - \frac{3}{n}$.
The main idea of our analysis is to show inductively that the critical position increases with high probability after $O(\mu \log n)$ iterations (\Cref{lem:increaseCriticalFrequency}).
Since it can be increased at most~$n$ times, the frequency vector is close to the all-$1$s vector with high probability after $O(\mu n \log n)$ iterations.
Afterward, the optimum is sampled with high probability within $O(\log n)$ iterations, yielding our main result.

In the following, we formalize these steps in different lemmas before we prove our main result.

The following lemma shows that the impact of genetic drift is low, implying that all frequencies remain with high probability greater than~$\frac{1}{4}$ for the duration of the optimization.

\begin{lemma}
  \label{lem:frequenciesDoNotDropTooLow}
  Consider the cGA maximizing \LO with $\mu \geq 786\exp(6)n \ln^2 n$.
  Then the probability that all frequencies are greater than $\frac{1}{4}$ in the first $12\exp(6)\mu n \ln n$ iterations is at least $1 - \frac{2}{n}$.
\end{lemma}
\begin{proof}
  We aim at applying \Cref{thm:geneticDrift} for $T \coloneqq 12\exp(6)\mu n \ln n$.
  We observe that \LO weakly prefers~$1$s over~$0$s for each position.
  Thus, \Cref{thm:geneticDrift} yields
  \begin{align*}
    \Pr \left[ \forall t \in [0..T]\colon p_i^{(t)} > \frac{1}{4} \right] \geq 1 - 2\exp \left ( \frac{-\mu}{384\exp(6) n \ln n} \right )
  \end{align*}
  for all $i \in [n]$.
  Hence, as $\mu \geq 768 \exp(6) n \ln ^2n$, we observe that
  \begin{align*}
    \Pr \left[ \forall t \in [0..T]\colon p_i^{(t)} > \frac{1}{4} \right] & \geq 1 - 2\exp \left ( -2\ln n \right ) \\
                                                                          & =  1 - \frac{2}{n^2}.
  \end{align*}
  The result follows by applying the union bound over all~$n$ frequencies,  hence adding their respective failure probabilities of~$\frac{2}{n^2}$.
\end{proof}

Next, we show that if a frequency~$p_i$ is at its maximum value of $1 - \frac{1}{n}$ and all preceding frequencies~$p_j$ with $j \in [i - 1]$ are at similarly high values, then~$p_i$ also remains, with high probability, close to its maximum value for the duration of the optimization.

\begin{lemma}
  \label{lem:frequencyStayingHigh}
  Consider the cGA maximizing \LO with $n \geq 8$ and $\mu \geq 786\exp(6)n \ln^2 n$, and let $i \in [2 .. n]$.
  Assume that at some iteration  $t_0$ the frequency at position~$i$ is $1 - \frac{1}{n}$ and the first $i - 1$ frequencies stay above $1 - \frac{3}{n}$  during the next $T_0 \coloneqq 12\exp(6)(2n - i + 1)\mu \ln n$ iterations. Then, with a probability of at least $1 - \frac{1}{n^3}$, for the next $T_0$ iterations, the frequency at position~$i$ stays above $1 -\frac{3}{n}$.
\end{lemma}

The proof of this lemma is somewhat technical as it needs a clever application of the negative drift theorem (\Cref{thm:negativeDrift}) to the value of the $i$-th frequency. To be able to apply this drift theorem and to obtain the best possible estimates, we twice modify the stochastic process. The first modification is of a technical nature and only changes the process so that it also exhibits the desired properties after the time interval in which the $i$-th frequency is above $1-\frac 3n$. This does not change the length of this time interval, so it does not influence our estimates, but this is necessary since the negative drift theorem requires the conditions to be fulfilled at each time step.

The second modification is more interesting. We note that when the $i$-th frequency is close to $1 - \frac 1n$, then with high probability it does not change, simply because both samples contain a one in the $i$-th position. This high rate of iterations with no change reduces the drift, but has little other effect. Consequently, by regarding the process only consisting of steps in which the $i$-th frequency changes, we obtain an essentially identical process, but with much stronger drift (this is the absolute value of $\varepsilon$ in \Cref{thm:negativeDrift}); for this process, the negative drift theorem gives much stronger estimates.

We note that these challenges did not appear in the analysis of how the UMDA optimizes \leadingones~\cite{DoerrK21tcs},  because there the higher selection rate (which could be obtained from the sample size larger than two) ensured that frequencies to the left of the critical frequency always stayed at $1-\frac 1n$.

\begin{proof}
  To ease the notation, we assume that $t_0 = 0$, that is, we artificially start the cGA with the situation assumed in the lemma at time $t_0$. Our goal is to apply \Cref{thm:negativeDrift} to the $i$-th frequency.
  We denote the stochastic process describing the run of the original cGA by $P$. Let $T$ be the number of iterations until the first time when the $i$-th frequency hits or goes below $1 - \frac{3}{n}$.

  We first modify $P$ so that we can show the assumptions of \Cref{thm:negativeDrift} also after time $T_0$. This modified process $P'$ is defined as follows. Until iteration $T_0$, the processes $P$ and $P'$ are identical. Afterwards, the frequency update of the cGA -- different from its original definition -- does not change the first $i-1$ frequencies, but only updates the $i$-th frequency and the higher ones.  Let $T'$ be the hitting time in $P'$ that corresponds to $T$ in $P$ (that is, the first time the $i$-th frequency in $P'$ hits or goes below $1-\frac 3n$).
  As for the first $T_0$ iterations $P$ and $P'$ act identically, we have $\Pr[T' \le T_0] = \Pr[T \le T_0]$. On the positive side, the process $P'$ by construction satisfied that the first $i-1$ frequencies always stay above $1-\frac 3n$.

  To increase the drift, that is, the absolute value of $\varepsilon$ in the application of \Cref{thm:negativeDrift}, we modify the process a second time by ignoring time steps in which the $i$-th frequency does not change in $P'$. Formally speaking, the process $P''$ consists of the initial state of $P'$ and then only of those states of $P'$ in which the $i$-th frequency has a value different from the previous state.
  Let $T''$ be the corresponding hitting time in $P''$ and $p''$ the corresponding frequency vector. Since $P''$ is a subsequence of the states of $P'$ which contains all states in which for the first time a particular value of the $i$-th frequency is reached, we have $T'' \le T'$ and hence
  $\Pr[T' \leq T_0] \leq \Pr[T'' \leq T_0]$.
  Thus, we aim to prove a probabilistic bound on $T''$.

  To fit our process to the setting of \Cref{thm:negativeDrift}, we define $(Y)_t$ as $Y_{t} = 1 - p_i^{\prime \prime(t)} - \frac{1}{n} -\frac{1}{\mu}$.
  We aim at showing that the probability of $(Y)$ hitting a state $\frac{2}{n} -\frac{1}{\mu}$ within $T_0$ iterations is smaller than $\frac{1}{n^3}$.
  To that end, we use the framework from \Cref{thm:negativeDrift}.
  In our setting, the variable $a$ from the theorem statement is 0 and the variable $b$ is $\frac{2}{n} - \frac{1}{\mu}$.
  We verify that the conditions~(a) to (c) hold.
  \begin{enumerate}
    \item[(a)] We show that there exists $\varepsilon < 0$ such that for all points in time $t < T''$ we have \[
            \E[(Y_{t + 1} - Y_{t}) \mid Y_{t} ] \cdot \mathbf{1}\{Y_{t} \geq 0 \} \leq \varepsilon \cdot \mathbf{1} \{Y_{t} \geq 0\}.
          \]
          For $Y_{t} < 0$, the inequality reduces to $0 \leq 0$. We now  focus on times $t$ when $Y_{t} \geq 0$. In those cases, $(Y)$ changes either by $\frac{1}{\mu}$ or $-\frac{1}{\mu}$. Therefore, the search points sampled by the cGA at the iteration corresponding to point $t$ in the modified process differ at the $i$-th index. Also, any bias to the update is introduced only when the first $i -1$ bits in both search points are ones. Otherwise, it is equally probable to increment or decrement the $i$-th frequency, as for each pair of search points that increments this frequency there exists a pair of search points that decrements this frequency. Thus
          \begin{align*}
            \E[Y_{t + 1} - Y_{t}\mid Y_{t}]
             & = -\frac{1}{\mu}\prod_{j = 1}^{i - 1}\left (p_j^{\prime \prime(t)} \right )^2 \\
             & \leq -\frac{1}{\mu}\left(1 - \frac{3}{n}\right) ^ {2(i - 1)}                  \\
             & \leq -\frac{1}{\mu}\left(1 - \frac{3}{n}\right) ^ {2(n - 1)}.
          \end{align*}
          We observe that the last expression decreases with $n$ increasing. Hence for $n \geq 8$, we have
          \begin{align*}
             & -\frac{1}{\mu}\left(1 - \frac{3}{n}\right) ^ {2(n-1)} \leq -\frac{1}{\mu}\left(1 - \frac{3}{8}\right) ^ {14} \leq - \frac{\exp(-6)}{2\mu}.
          \end{align*}
          Hence it suffices to set the value to of $\varepsilon$ to $-\frac{\exp({-6})}{2\mu}$ to satisfy condition~(a).
    \item[(b)] We show that there exists a constant $c$ such that for all points in time $t < T''$, we have $|Y_{t} - Y_{ t + 1}| \cdot \mathbf{1}\{Y_{t} \geq 0\} < c$. For times $t$ when $Y_{t} < 0$, the inequality reduces to $0 < c$. For times $t$ when $Y_t \geq 0$, at each discrete step $(Y)$ changes either by $\frac{1}{\mu}$ or $ -\frac{1}{\mu}$, so setting $c$ to $\frac{2}{\mu}$ suffices.
    \item[(c)] We show that for all points in time $t < T''$, we have $Y_{ t+ 1} \cdot \mathbf{1} \{Y_{t} < 0 \} \leq 0$. Observe that if $Y_{t} <0$, then $Y_{t} = -\frac{1}{\mu}$. Thus, $Y_{ t + 1} \leq 0$, as a frequency changes at most by $\frac{1}{\mu}$ in each iteration.
  \end{enumerate}
  With conditions~(a) to (c), \Cref{thm:negativeDrift} yields
  \begin{align*}
     & \Pr[T''\leq T_0] \leq T_0^2 \exp \left( -\frac{\mu\left(\frac{2}{n} - \frac{1}{\mu}\right)\exp({-6})}{16} \right)        \\
     & \quad\leq (12\exp(6)(2n - i + 1)\mu \ln n)^2                                                                             \\
     & \hspace*{3.5 cm}\cdot\exp \left( -\frac{\mu\left(\frac{2}{n} - \frac{1}{\mu}\right)\exp({-6})}{16} \right)               \\
     & \quad\leq (24\exp(6)\mu n\ln n)^2 \exp \left( -\frac{\mu\left(\frac{2}{n} - \frac{1}{\mu}\right)\exp({-6})}{16} \right).
  \end{align*}
  As $\mu > n$, we obtain
  \begin{align*}
     & \Pr[T''\leq T_0]                                                                  \\
     & \quad\leq (24\exp(6)\mu n \ln n)^2 \exp \left( -\frac{\mu \exp(-6)}{16n}\right)   \\
     & \quad= 576\exp(12)\mu^2 n^2 \ln^2 n \exp \left( -\frac{\mu \exp(-6)}{16n}\right).
  \end{align*}

  Thus considering $\mu \geq 786\exp(6)n \ln^2 n$ we get
  \begin{align*}
     & \Pr[T''\leq T_0]                                                                     \\
     & \quad\leq 576\cdot 786^2 \exp(24) n^4 \ln^6 n \exp \left ( - 49.125 \ln^2 n \right ) \\
     & \quad\leq 576\cdot 786^2 \exp(24) n^4 \ln^6 n \exp \left ( - 49 \ln^2 n \right )
  \end{align*}
  Since  $n \geq 8$ we have $\ln n > 1$, we can estimate
  \begin{align*}
     & \Pr[T''\leq T_0]                                                               \\
     & \quad\leq 576\cdot 786^2 \exp(24) n^4 \ln^6 n \exp \left ( - 49 \ln n \right ) \\
     & \quad\leq \frac{576\cdot 786^2 \exp(24) n^4 \ln^6 n}{n^{49}},
  \end{align*}
  which is easily seen to be smaller than $\frac{1}{n^3}$ for $n \geq 8$. As we showed in the beginning, the same probability bound holds for~$T$, which concludes the proof.
\end{proof}

The next lemma shows that the frequency of the critical position reaches its maximum value within $O(\mu \log n)$ iterations. This is again very different from the UMDA, where the larger sample size and the higher selection rate allowed to move this frequency to the desired value in a single iteration~\cite{DoerrK21tcs}.

The key to the proof of this lemma is observing that the distance of the critical frequency to the ideal value of one exhibits a multiplicative drift behavior, that is,  the frequency moves up with expected speed proportional to the current distance (when, as assumed here, the frequency is at least some constant, here~$\frac 14$).

\begin{lemma}
  \label{lem:increaseCriticalFrequency}
  Consider the cGA maximizing \LO with $n \geq 8$ and $\mu \geq 786\exp(6)n \ln^2 n$, and let $i \in [0 .. n - 1]$.
  Assume that at some iteration $t_0$ the frequency at position $i + 1$ is at least $\frac{1}{4}$, the first $i$ frequencies are greater than $1 - \frac{3}{n}$, and this condition also holds for the next $T_0 \coloneqq 12\exp(6)\mu \ln n$ iterations. Then, with a probability of at least $1 -\frac{1}{n^2}$, the frequency at position $i + 1$ reaches $1 - \frac{1}{n}$ in the next $T_0$ iterations.
\end{lemma}

\begin{proof}
  Again, to ease the notation we assume that $t_0$ is zero, that is, the process is artificially started in the state assumed in the lemma.

  Our goal is to apply \Cref{thm:multiplicativeDrift} to the $(i+1)$-st frequency. We denote the process corresponding to the cGA as $P$. Let $T$ be the number of iterations until the first time when the $(i + 1)$-st frequency hits $1-\frac{1}{n}$.
  As in the previous lemma, to use \Cref{thm:multiplicativeDrift} we need that our process exhibits a suitable drift throughout its runtime and not just for $T_0$ iterations. Hence as there, we artificially modify the process to ensure this property, and this in a way that it does not interfere with what we want to show. Hence let $P'$ be the following process. $P$ and $P'$ are identical until iteration $T_0$. Afterwards the first $i$ frequencies in $P'$ do not change anymore, whereas the $(i+1)$-st frequency and the higher frequencies are updated as in the definition of the cGA. Let $T'$ be the corresponding first-hitting time in $P'$ and $p'$ the corresponding frequency vector in $P'$.  Observe that as for the first $T_0$ iterations $P$ and $P'$ act identically, $\Pr[T' < T_0] = \Pr[T < T_0]$. Thus we now prove a probabilistic bound on $T'$.

  To fit our process to the setting of \Cref{thm:multiplicativeDrift}, we define a sequence of random variables $(X)_t$ by $X_t = 1 - p^{\prime(t)}_{i + 1}$ if $p^{\prime(t)}_{i + 1} \neq 1-\frac 1n$ and $X_t = 0$ otherwise. With this definition, $T'$ is the first time this sequence hits zero. We now show that this process exhibits a multiplicative drift as required in \Cref{thm:multiplicativeDrift}.

  Let $x_1, x_2$ be the offspring sampled by the cGA in some iteration and $y_1, y_2$ the offspring after the possible swap. As in \Cref{lem:frequencyStayingHigh}, the expected change of the $(i + 1)$-st frequency is non-zero only if both $x_1$ and $x_2$ start with $i$ ones and differ in the $(i+1)$-st position.
  Recalling that $X_{i + 1,t} = 1 - p_{i + 1}^{\prime(t)}$, we deduce that
  \begin{align*}
     & \E[X_{i + 1,t} - X_{i + 1, t + 1} \mid X_{i + 1,t}]                                                                          \\
     & \quad= \frac{2}{\mu} \left ( \prod_{j = 1}^i (p_j^{\prime(t)})^2 \right ) p_{i + 1}^{\prime(t)}(1 - {p}_{i + 1}^{\prime(t)}) \\
     & \quad= \frac{2X_{i + 1,t}}{\mu}\left ( \prod_{j = 1}^i (p_j^{\prime(t)})^2 \right ) p_{i + 1}^{\prime(t)}                    \\
     & \quad\geq \frac{X_{i + 1,t}}{2\mu} \left ( 1 - \frac{3}{n} \right )^{2i}.
  \end{align*}
  As in \Cref{lem:frequencyStayingHigh}, we use the fact that for $n \geq 8$, we have $(1 - \frac{3}{n})^{2(n - 1)} > \frac{\exp(-6)}{2}$, and obtain
  \begin{align*}
     & \E[X_{i + 1,t} - X_{i + 1, t + 1} \mid X_{i + 1,t}] \\
     & \quad\geq \frac{X_{i + 1,t} \exp({-6})}{4\mu}.
  \end{align*}
  We apply Theorem \ref{thm:multiplicativeDrift} and obtain, for any $r \in \R_{\ge 0}$, that
  \begin{align*}
    \Pr \left [ T' > 4\mu \left (\frac{r + \ln \left ( \frac{3n}{4}\right )}{\exp(-6)} \right ) + 1 \right ]
     & \leq \exp(-r).
  \end{align*}
  Setting $r \coloneqq 2 \ln n$ we get
  \begin{align*}
    \frac{1}{n^2} & \geq \Pr \left [ T' > 4\mu \left (\frac{2\ln(n) + \ln \left ( \frac{3n}{4}\right ) + \frac{\exp(-6)}{4\mu}}{\exp(-6)} \right )  \right ] \\
                  & \geq \Pr \left [T' >  4\mu \left(\frac{3\ln n}{\exp(-6)}\right)\right] = \Pr[T' >T_0].
  \end{align*}
  As we showed before, the same probabilistic bound holds for~$T$. This finishes the proof.
\end{proof}

When counting function evaluations,  the result above shows a slighter slower model update for the cGA than what could be shown for the UMDA in~\cite{DoerrK21tcs}. Above, we have seen that the cGA within $O(\mu \log n)$ iterations, hence with $O(\mu \log n)$ samples, moves the critical frequency to the upper border. In~\cite{DoerrK21tcs}, it was shown that a single iteration, hence $\Theta(\mu)$ samples, suffices for the UMDA with $\mu = \Omega(n \log n)$ and $\lambda = \Theta(\mu)$ a suitable constant factor larger than $\mu$. While the roles of $\mu$ are not perfectly comparable in both algorithms, they are quite similar, in particular, the same value of $\mu$ in both algorithms leads to comparable genetic drift. For the particular results shown here and in~\cite{DoerrK21tcs}, the UMDA can even work with a slightly smaller value of $\mu$, namely $\mu=\Theta(n \log n)$, whereas we assume $\mu = \Omega(n \log^2 n)$. This difference is caused by the longer runtime needed on our result, which requires a weaker genetic drift.

Overall, our analysis suggests that the model update of the cGA allows for slightly slower model adjustments. From looking at our proof, we feel that the true reason for this effect is that fact that the cGA changes a frequency only when it samples two different values. Hence moving a frequency from very close to the optimal value to the optimal value itself takes many samples for the cGA, whereas the UDMA does this modification very easily.

Last, we combine \Cref{lem:frequenciesDoNotDropTooLow,lem:frequencyStayingHigh,lem:increaseCriticalFrequency} inductively and prove our main result.

\begin{proof}[Proof of \Cref{thm:cgaOnLO}]
  We aim at showing that with high probability the critical position increases at least each $12\exp(6) \mu \ln n$ iterations, and that when all frequencies are greater than $1 - \frac{3}{n}$,  the optimum is sampled quickly. We proceed by induction on the frequencies. Formally, we prove that for every frequency (indexed by $i \in [n]$), with a probability of at least $(1-\frac{2}{n})(1-\frac{1}{n^2})^{i}(1 - \frac{1}{n^3})^{i}$, there exists an iteration $T_i\leq 12\exp(6)i \mu \ln n$ for which
  \begin{enumerate}
    \item the first $i - 1$ frequencies are greater than $1- \frac{3}{n}$,
    \item the $i$-th frequency is $1 - \frac{1}{n}$, and
    \item\label[prop]{item:otherFrequenciesRemainHigh} the first $i$ frequencies remain greater than $1 - \frac{3}{n}$ for the next $12\exp(6)(2n -i + 1)\mu \ln n$ iterations.
    \item All the frequencies remain greater than $\frac{1}{4}$ for the first $12\exp(6)\mu n \ln n$ iterations.
  \end{enumerate}
  We observe that the fourth property is a direct result of \Cref{lem:frequenciesDoNotDropTooLow} and holds with a probability of at least $1 - \frac{2}{n}$, regardless of the frequency considered. Thus, we do not mention it explicitly anymore but include it in the total probability.

  Consider the first frequency from the start of the algorithm on, that is, we regard the case $i=1$. By \Cref{lem:increaseCriticalFrequency}, it reaches $1 - \frac{1}{n}$ within $12\exp(6)\mu \ln n$ iterations with a probability of at least $1 - \frac{1}{n^2}$. At the iteration when it reaches $1 - \frac{1}{n}$, by \Cref{lem:frequencyStayingHigh}, it remains greater than $1 - \frac{3}{n}$ for the next $12\exp(6)(2n - i + 1)\mu \ln n$ iterations with a probability of at least $1 - \frac{1}{n^3}$. Thus, this iteration fulfills all properties to be~$T_1$.

  We assume that for some $i \in [n - 1]$, $T_i$ exists with a probability of at least $ (1 - \frac{2}{n})(1- \frac{1}{n^2})^i(1 - \frac{1}{n^3})^i$. We show that with a probability of at least $(1-\frac{2}{n})(1-\frac{1}{n^2})^{i + 1}(1 - \frac{1}{n^3})^{i + 1}$ $T_{i + 1}$ exists. By \Cref{item:otherFrequenciesRemainHigh} of $T_i$, the conditions for \Cref{lem:increaseCriticalFrequency} are fulfilled at iteration $T_i$, thus with a probability of at least $1 - \frac{1}{n^2}$, so a total probability of at least $(1-\frac{2}{n})(1-\frac{1}{n^2})^{i + 1}(1 - \frac{1}{n^3})^{i}$, the $i+1$-st frequency reaches $1- \frac{1}{n}$ within the next $12\exp(6)\mu \ln n$ iterations. Furthermore, when that happens, by \Cref{lem:frequencyStayingHigh}, with a probability of at least $1 - \frac{1}{n^3}$, so a total probability of $(1-\frac{2}{n})(1-\frac{1}{n^2})^{i + 1}(1 - \frac{1}{n^3})^{i + 1}$, the first $i + 1$ frequencies remain greater than $1 - \frac{3}{n}$ for the next $12\exp(6)(2n -i)\mu \ln n$ iterations. Thus, this time is the desired iteration $T_{i +1}$:
  \begin{enumerate}
    \item the first $i$ frequencies are greater than $1 - \frac{3}{n}$ due to \Cref{item:otherFrequenciesRemainHigh} for $T_i$,
    \item the $i + 1$-st frequency is equal to $1 - \frac{1}{n}$ by definition,
    \item the first $i + 1$ frequencies remain greater than $1 - \frac{3}{n}$ for the next $12\exp(6)(2n - i)\mu \ln n$ iterations.
  \end{enumerate}

  Therefore, by induction, with a probability of at least $(1 - \frac{2}{n})(1 - \frac{1}{n^2})^n(1 - \frac{1}{n^3})^n$, there exists an iteration $T_n \leq 12\exp(6)\mu n \ln n$.

  We observe that from iteration $T_n$ onwards, it is guaranteed that for another $12\exp(6)(n +1)\mu \ln n$ iterations all frequencies remain greater than $1 - \frac{3}{n}$, i.e., the probabilistic model of the cGA is very close to the optimum. By applying Bernoulli's inequality multiple times we obtain a lower bound on the derived probability: $(1 - \frac{2}{n})(1 - \frac{1}{n^2})^n(1 - \frac{1}{n^3})^n \geq (1 - \frac{2}{n})(1 - \frac{1}{n})(1 - \frac{1}{n^2}) \geq (1 - \frac{2}{n})(1 - \frac{1}{n})^2 \geq (1 - \frac{2}{n})^2 \geq 1- \frac{4}{n}$. This finishes the proof of the first statement in the theorem. To prove the second statement, we observe that at each iteration in which all frequencies are at least $1 - \frac 3n$, the probability of sampling an optimum is at least $(1-\frac3n)^n \ge (1-\frac38)^8 \ge 0.02$, using $n \ge 8$ again, so the probability of sampling an optimum in the next
  $50 \ln n \leq12\exp(6)(n +1)\mu \ln n$ iterations is at least $1 - (1-0.02)^{50 \ln n} \ge 1 - e^{-\ln n} = 1 - \frac 1n$. This yields a total success probability of at least $(1- \frac{4}{n})(1 - \frac{1}{n}) \geq 1 - \frac{5}{n}$. This finishes the proof.
\end{proof}

\textbf{Lessons from the analysis:}
We observe that the proof of our main result is significantly more involved than the proof of the corresponding result for the UMDA~\cite{DoerrK21tcs}. The reason is that indeed the optimization process is more complex for the cGA. To see this, consider the situation that the current frequency vector $p$ is such that the first $\ell = \Theta(n)$, $\ell < n$, frequencies are equal to $1-\frac 1n$, that is, they are at the upper border. In this situation, a sample $x$ has $\ell$ leading ones with probability $(1-\frac 1n)^\ell \ge (1-\frac 1n)^{n-1} \ge \frac 1e$. Consequently, when the sample size of the UMDA is sufficiently large (and this is necessary anyway to be in the low-genetic-drift regime) and the selection rate is a constant smaller than $\frac 1e$, the frequency update of the UMDA with high probability only depends on samples with~$\ell$ leading ones, that is, the new frequency vector again has the first~$\ell$ frequencies at the upper border. Consequently, an initial segment of frequencies at the upper border stays at this border for a long time, and the runtime analysis only has to care about how the critical frequency is moving to the upper border of the frequency range.

In contrast, due to the small sample size of two, the cGA does not observe such a stable optimization behavior. In fact, in the above situation, we see that with constant probability, the two offspring sampled by the cGA are such that both have a single zero in the first $\ell$ bits, and these are at different locations $i < j$. In this case, the frequency update favors the bit values of the sample having the zero on position $j$, which means that the $j$-th frequency is moving from $1-\frac 1n$ to $1- \frac 1n - \frac 1\mu$. Hence with constant probability, a frequency of the previously perfect initial segment is moving into the wrong direction. In consequence, these frequencies do not all stay at the upper border, but a constant fraction of them is strictly below it.

We cope with this difficulty by showing that, despite the effect just described, there is still a drift of the frequencies towards the upper border, ensuring that a single frequency $p_i$ with high probability does not go below $1-\frac 3n$ (Lemma~\ref{lem:frequencyStayingHigh}). This drift only is strong enough when the previous $i-1$ frequencies are close to the upper border, which is why Lemma~\ref{lem:frequencyStayingHigh} takes this assumption and why our main proof requires the slightly technical induction.

The additional technicality of our proof, the slightly weaker bound we obtain, and the above considerations suggest that the UMDA with its larger sample size allows for a more stable and (slightly) more efficient optimization process. This could be a reason why the UMDA is more often used in practical applications.

At the same time, we note that our result also shows that the cGA can cope with this slightly less stable optimization process. Our bounds are by a polylogarithmic factor worse than the bounds for the UMDA, which is a small discrepancy, and given we have no matching lower bound, it is not even clear if there is a discrepancy at all.

\section{Conclusion}
\label{sec:conclusion}

This work conducts the first runtime analysis for the cGA on the \leadingones benchmark, which was a result so far missing among the runtime analyses of classic univariate EDAs on classic, simple benchmark problems. Our analysis detects some additional difficulties compared to the optimization of this benchmark via the UMDA. We overcome these with a more complex proof and finally show a runtime guarantee only slightly inferior to the one of the UMDA. On the positive side, this result shows that also the simple cGA (having a single parameter only) is able to optimize \leadingones. However, our result also suggests that the smaller sample size of the cGA can be a disadvantage, and that more complex algorithms like the UMDA (having two parameters) might be more successful (with the right parameters).

From comparing our proof with the one for the UMDA in~\cite{DoerrK21tcs}, we observe some structural differences in the working principles of the cGA and the UMDA. This is noteworthy since in most existing results, the two algorithms seemed to behave very similarly.

The biggest open problem stemming from this work is proving a matching lower bound for our result, which is generally a hard problem for EDAs. Such a bound would answer the question whether the differences in the algorithms truly lead to a different asymptotic runtime (with the difference at most being a polylogarithmic factor), or whether the two algorithms despite different working principles actually are equally efficient.

\bibliographystyle{IEEEtran}
\bibliography{ich_master,alles_ea_master,rest}

\end{document}